\documentclass[sigconf]{aamas}
\AtBeginDocument{
  \providecommand\BibTeX{{
    \normalfont B\kern-0.5em{\scshape i\kern-0.25em b}\kern-0.8em\TeX}}}

\usepackage{flushend}
\setcopyright{acmcopyright}
\setcopyright{ifaamas}
\copyrightyear{2020}
\acmYear{2020}
\acmDOI{}
\acmPrice{}
\acmISBN{}
\acmConference[AAMAS'20]{Proc.\@ of the 19th International Conference on Autonomous Agents and Multiagent Systems (AAMAS 2020)}{May 9--13, 2020}{Auckland, New Zealand}{B.~An, N.~Yorke-Smith, A.~El~Fallah~Seghrouchni, G.~Sukthankar (eds.)}

\usepackage{booktabs}
\usepackage{amsfonts}
\usepackage{nicefrac}
\usepackage{breqn}
\usepackage{microtype}

\usepackage{amsmath, amsfonts, amsthm}

\usepackage{mathtools}

\DeclarePairedDelimiter{\abs}{\lvert}{\rvert}

\newcommand{\norm}[1]{\left\lVert#1\right\rVert}

\newcommand{\reals}{\mathbb{R}}
\newcommand{\expectation}{\mathbb{E}}
\newcommand{\subex}[1]{\left(#1\right)}
\newcommand{\subblock}[1]{\left[#1\right]}
\newcommand{\tuple}{\subex}
\newcommand{\defword}[1]{\textbf{\boldmath{#1}}}
\newcommand{\set}[1]{\left\{ {#1} \right\}}

\newcommand{\as}{\doteq}

\newcommand{\algSimple}{L}
\newcommand{\alg}[1][t]{L_{#1}(h)}
\newcommand{\ylinkSimple}[1][\Phi]{Y^{#1}}
\newcommand{\ylink}[2][\Phi]{Y_{#2}^{#1}}
\newcommand{\Actions}{A}
\newcommand{\reward}{r}

\newcommand{\regret}{\rho}
\newcommand{\Regret}{R}
\newcommand{\linkFn}{f}

\newcommand{\gap}{\varepsilon}
\newcommand{\altGap}{\epsilon}

\newcommand{\supReward}{U}
\newcommand{\radius}{2 \supReward}
\newcommand{\ylinkEstSimple}[1][\Phi]{\Tilde{Y}^{#1}}
\newcommand{\ylinkEst}[2][\Phi]{\Tilde{Y}_{#2}^{#1}}

\newcommand{\reachProb}{\eta}
\newcommand{\policy}{\sigma}
\newcommand{\PolicySpace}{\Sigma}
\newcommand{\chance}{c}
\newcommand{\chancePolicy}{\sigma_{\chance}}

\newcommand{\history}{h}
\newcommand{\Histories}{\mathcal{H}}
\newcommand{\InfoStates}{\mathcal{S}}
\newcommand{\infoState}{s}
\newcommand{\TerminalHistories}{\mathcal{Z}}

\newcommand{\player}{i}
\newcommand{\playerChoice}{p}
\newcommand{\ActionSet}{\mathcal{A}}
\newcommand{\ev}{\reward}

\newcommand{\cfq}{v}
\newcommand{\cfr}{\regret}
\newcommand{\Cfr}{\Regret}
\newcommand{\FullRegret}{R^{\text{EXT}}}

\newcommand{\featureExp}{\varphi}

\usepackage{upgreek}
\newcommand{\FeatureSpace}{\reals^d}

\newcommand{\functionApproximator}{y}

 \makeatletter
    \@ifdefinable{\epsGenBlackwellCond}{\def\epsGenBlackwellCond/{$(\Phi, f, \epsilon)$-Blackwell condition}}
\makeatother
\makeatletter
    \@ifdefinable{\epsGenBlackwellCondTitle}{\def\epsGenBlackwellCondTitle/{$(\Phi, f, \epsilon)$-Blackwell Condition}}
\makeatother
\makeatletter
    \@ifdefinable{\eg}{\def\eg/{\emph{e.g.}}}
\makeatother
\makeatletter
    \@ifdefinable{\ie}{\def\ie/{\emph{i.e.}}}
\makeatother
\makeatletter
    \@ifdefinable{\etal}{\def\etal/{\emph{et al.}}}
\makeatother
\makeatletter
    \@ifdefinable{\Politex}{\def\Politex/{\textsc{Politex}}}
\makeatother

\begin{document}

\title[Alternative Function Approximation Parameterizations for Solving Games]{Alternative Function Approximation Parameterizations for Solving Games:\\An Analysis of $f$-Regression Counterfactual Regret Minimization}
\subtitle{}

\author{Ryan D'Orazio}
\authornote{Department of Computing Science, University of Alberta. Edmonton, Alberta, Canada.}
\affiliation{}
\authornote{Equal contributors.}
\email{rdorazio@ualberta.ca}

\author{Dustin Morrill}
\authornotemark[1]
\affiliation{}
\authornotemark[2]
\email{morrill@ualberta.ca}

\author{James R. Wright}
\affiliation{}
\authornotemark[1]
\email{james.wright@ualberta.ca}

\author{Michael Bowling}
\authornotemark[1]
\affiliation{}
\email{mbowling@ualberta.ca}

\begin{abstract}
Function approximation is a powerful approach for structuring large decision problems that has facilitated great achievements in the areas of reinforcement learning and game playing. Regression counterfactual regret minimization (RCFR) is a simple algorithm for approximately solving imperfect information games with normalized rectified linear unit (ReLU) parameterized policies. In contrast, the more conventional softmax parameterization is standard in the field of reinforcement learning and yields a regret bound with a better dependence on the number of actions. We derive approximation error-aware regret bounds for $(\Phi, f)$-regret matching, which applies to a general class of link functions and regret objectives. These bounds
recover a tighter bound for RCFR and provide a
theoretical justification for RCFR implementations with alternative policy parameterizations ($f$-RCFR), including softmax. We provide exploitability bounds for $f$-RCFR with the polynomial and exponential link functions in zero-sum imperfect information games and examine empirically how the link function interacts with the severity of the approximation. We find that the previously studied ReLU parameterization performs better when the approximation error is small while the softmax parameterization can perform better when the approximation error is large.
 \end{abstract}



\maketitle

\section{Introduction}

The dominant framework for approximating Nash equilibria
in sequential games with imperfect information is \defword{Counterfactual Regret Minimization (CFR)}, which has successfully been used to solve and expertly play human-scale poker games~\cite{bowling2015heads,moravvcik2017deepstack,brown2018superhuman,brown2019superhuman}. This framework is built on the idea of decomposing a game into a network of simple regret minimizers~\cite{zinkevich2008regret,farina2019regret}.
Historically, large games have been abstracted to smaller but strategically similar games through a state-aggregation procedure~\citep{zinkevich2008regret, waugh2009abstraction, johanson2013evaluating, ganzfried2013action}. The abstract game is solved with CFR and the resulting strategies are translated so they apply to the original game.

Function approximation is a natural generalization of abstraction.
In CFR, this amounts to estimating the regrets for each regret minimizer instead of storing them all in a table~\citep{waugh2015solving,morrill2016,deepCFR,li2018doubleNeuralCfr,steinberger2019single}.
Game solving with function approximation can be competitive with domain specific state abstraction~\cite{waugh2015solving,morrill2016,deepCFR,heinrich2016deep}, and
in some cases is able to outperform tabular CFR without abstraction if the players
are optimizing against their best responses~\cite{Exp_Descent}. Function approximation has facilitated many recent successes in game playing more broadly~\cite{DSilverHMGSDSAPL16AlphaGo,silver2018alphaZero,Vinyals19AlphaStar}.

Combining regression and regret-minimization with applications to CFR was initially studied by
Waugh \etal/~\cite{waugh2015solving}, introducing the \defword{Regression Regret-Matching (RRM)} Theorem---giving a sufficient
condition for function approximator error to still achieve no external regret. The extension to \defword{Regression Counterfactual Regret Minimization (RCFR)} yields an algorithm that utilizes function approximation in a way similar to reinforcement learning (RL), particularly policy-based RL. Action preferences---cumulative \defword{counterfactual regrets}---are learned and predicted, and these predictions parameterize a stochastic policy.

Conversely, some recent RL algorithms take a regret minimization approach. Regret policy gradient (RPG)~\cite{CFR_Actor_Critic}, exploitability descent (ED)~\cite{Exp_Descent}, \Politex/~\cite{abbasi2019politex}, and neural replicator dynamics (NeuRD)~\cite{neuRD} either have regret bounds or they are inspired by tabular algorithms with regret bounds.

CFR was originally introduced using \defword{regret matching (RM)}~\cite{Hart00Rm} as its component learners. This learning algorithm generates policies by normalizing positive regrets and setting the weight of actions with negative regrets to zero. This truncation of negative regrets is exactly the application of a rectified linear unit (ReLU) function, which is used extensively in the field of machine learning for constructing neural network layers. RCFR, following in CFR's lineage, had only theoretical guarantees with normalized ReLU policies.

However, most RL algorithms for discrete action spaces take a different approach: they exponentiate and normalize the preferences according to the \defword{softmax} function. The \defword{Hedge} or \defword{Exponential Weights} learning algorithm~\cite{freund1997decision} also uses a softmax function to generate policies. It even has a regret bound with a log dependence on the number of actions, rather than a square root dependence, as RM does. This provides us with some motivation for generalizing the RRM and RCFR theory to allow for alternative policy parameterizations.

In fact, RM and Hedge can be unified. Greenwald \etal/~\cite{greenwald2006bounds} present \defword{$(\Phi, f)$-regret matching}, a general framework for constructing learners
for minimizing $\Phi$-regret---a set of regret metrics that include
external regret and internal regret when using
a policy parameterized by a \defword{link function} $f$.
Generalizing to internal regret has an important connection to correlated
equilibria in general sum games~\cite{cesa2006prediction}.

In this paper, we first generalize the RRM Theorem to $(\Phi, f)$-regret matching by extending Greenwald \etal/~framework to the case when the regret inputs to algorithms are approximate. This new approximate $(\Phi, f)$-regret matching framework allows for the use of a broad class of link functions and regret objectives, and provides a simple recipe for generating regret bounds under new choices for both when estimating regrets.
Our analysis, both due to improvements previously made by Greenwald \etal/~\cite{greenwald2006bounds} and more careful application of conventional inequalities, tightens the bound for RRM. The corresponding improvement to the RCFR Theorem~\cite{waugh2015solving,morrill2016} is magnified because the bound in this theorem is essentially the RRM bound multiplied by the size of the game.
In addition, this framework provides insight into the effectiveness of
combining function approximation with regret minimization as the impact of inaccuracy on the bounds vary between link functions and parameter choices.

The approximate $(\Phi, f)$-regret matching framework provides the basis for bounds that apply to RCFR algorithms with alternative link functions, thereby allowing the sound use of alternative policy parameterizations, including softmax. We call this generalization, \defword{$f$-RCFR}. We provide regret and equilibrium approximation bounds for this algorithm with the polynomial and exponential link functions, and we test them in two games commonly used in games research, \defword{Leduc hold'em poker}~\cite{Southey05leduc} and \defword{imperfect information goofspiel}~\cite{lanctot13phdthesis}. A simple but extensible linear representation is used to isolate the effect of the link function and the degree of approximation on learning performance.
We find that the polynomial link function performs better when the approximation error is small while the exponential link function (corresponding to a softmax parameterization) can perform better when the approximation error is large.

This paper is organized as follows. First, we define online decision problems and connect to relevant prior work in this area.
We then define approximate regret matching and provide regret bounds for this new class of algorithms.
Afterward, we begin our discussion of RCFR and our new generalization by describing extensive-form games and prior work on RCFR.
Finally, we present $\linkFn$-RCFR along with exploitability bounds and experiments in Leduc hold'em and goofspiel.

\section{Online Decision Problems}
\subsection{Background}

We adopt the notation from Greenwald \etal/~\cite{greenwald2006bounds} to describe an
\defword{online decision problem (ODP)}.
An ODP consists of a set of possible actions $A$ and set of possible rewards $\mathcal{R}$.
In this paper we assume a finite set of actions and
bounded $\mathcal{R} \subset \reals$ where
$\sup_{x \in \mathcal{R}} |x| = \supReward$.
The tuple $(A, \mathcal{R})$ fully characterizes the problem and is referred to as a reward
system. Furthermore, let $\Pi$ denote the set of reward functions $r: A \to \mathcal{R}$.

At each round $t$ an agent selects a policy, that is, a distribution over actions $\policy_t \in \Delta(A)$\footnote{$\Delta(A)$ is the set of all probability distributions over actions in $A$.}.
The agent samples an action, $a_t \sim \policy_t$, and subsequently receives a reward function, $r_t \in \Pi$.
The agent is able to compute the rewards for actions that were not taken at time
$t$, in contrast to the bandit setting where the agent only observes $r_t(a_t)$.

Crucially, each $r_t$ may be selected arbitrarily from $\Pi$. As a consequence, this ODP model is flexible enough to encompass multi-agent, adversarial interactions, and game theoretic equilibrium concepts even though it is described from the perspective of a single agent's decisions.

A learning algorithm in an ODP selects $\policy_t$ using information
from the history of observations and actions previously
taken. We denote this information at time $t$ as history $h \in H_t \as A^t \times \Pi^t$, where $H_0 \as
\{\emptyset\}$. Formally, an online learning algorithm is a sequence of functions
$\{L_t \}_{t=1}^\infty$, where $L_t : H_{t-1} \to \Delta(A)$.

We denote the \defword{rectified linear unit (ReLU)} function as $x^+ = \max \set{x, 0}$,
for all $x \in \reals$. Similarly for vectors $x \in \reals^N$ we define
$x^+$ to be the componentwise application of the ReLU function.

\subsubsection{Action Transformations}

To generalize the analysis to different performance metrics, it is useful to define \defword{action transformations}.
Action transformations are functions of the form $\phi : A \to \Delta(A)$, mapping each action $a \in \Actions$ to a policy.
Let $\Phi_{ALL}$
denote the set of all action transformations for the set of actions $A$.
Two important subsets of $\Phi_{ALL}$ are the external and internal transformations.

External transformations, $\Phi_{EXT}$, transform all actions to the same action.
Formally, if $\delta_a \in \Delta(A)$ is the
distribution with full weight on action $a$, then
$\Phi_{EXT} \as \{\phi : \exists a \in A \, \forall x \in A \quad \phi(x) = \delta_a \}$.
Note that there are $|\Phi_{EXT}| =|A|$-external transformations.

Internal transformations, $\Phi_{INT}$, transform one action to another action. Formally,
the internal transformation from action $a$ to action $b$ is defined piecewise as
$\phi_{INT}^{(a,b)}(a) = \delta_b$ and $\phi_{INT}^{(a,b)}(x) = \delta_{x}$ otherwise.
Note that there are $|\Phi_{INT}| = |A|^2 - |A| + 1$-internal transformations~\cite{greenwald2006bounds}.

We define the policy induced by distribution $\policy$ and action transformation $\phi$ as $[\phi](\policy) = \sum_{a \in A}\policy(a)\phi(a)$.

\subsubsection{Regret}

The regret for not following action transformation $\phi$
when action $a$ was chosen and reward function $r$ was observed is
\defword{$\phi$-regret}, $\rho^\phi(a,r) = \expectation_{a' \sim \phi(a)}[r(a')] - r(a)$.
For a set of action transformations, $\Phi$, the $\Phi$-regret vector is
$\rho^\Phi(a,r) = (\rho^\phi(a,r))_{\phi \in \Phi}$. The expected
$\phi$-regret for policy $\policy \in \Delta(A)$ is $\expectation_{a \sim \policy}[\rho^\phi (a,r)]$.

For an ODP with observed history $h$ at time $t$,
composed of reward functions $\{ r_s\}_{s=1}^t$ and
actions $\{a_s\}_{s=1}^t$ selected by the agent on each round,
the cumulative $\Phi$-regret after $t$-rounds
against action transformations $\Phi$ is $R^\Phi_t(h) = \sum_{k=1}^t{\rho^\Phi(a_k,r_k)}$.
For brevity we will omit the $h$ argument, and for convenience we set $R^\Phi_0 \as 0$.

We seek to bound the expected average maximum $\Phi$-regret,
$\expectation[ \frac{1}{t} \max_{\phi \in \Phi} R^\phi_t ]$.
Choosing $\Phi$ to be $\Phi_{EXT}$ or $\Phi_{INT}$
corresponds to the well studied maximum external regret or maximum internal regret objectives, respectively.

One can also interchange the max and the expectation.
In RRM, $ \max_{\phi \in \Phi_{EXT}} \expectation [ \frac{1}{t}  R^\phi_t ]$ is bounded~\citep{waugh2015solving, morrill2016}. However, bounds for
\linebreak[4]$\expectation[ \frac{1}{t} \max_{\phi \in \Phi} R^\phi_t ]$
imply similar bounds
when the expected regret, $\expectation[ \regret^{\Phi}_t]$, is observed after each round~\cite[Corollary 18]{greenwald2006bounds}.
The bounds remain the same with the exception of
replacing the observed random regrets with their corresponding expected values.

\subsection{Approximate Regret-Matching}

Given a set of action transformations $\Phi$ and a link function
$f : \reals^{|\Phi|} \to \reals^{|\Phi|}_+ $,
where $\reals^N_+$ denotes the $N$-dimensional positive orthant, we can define a general class of online learning
algorithms known as \defword{$(\Phi, f)$-regret-matching} algorithms~\cite{greenwald2006bounds}.
A $(\Phi, f)$-regret-matching algorithm at time $t$
chooses $\policy \in \Delta(A)$ that is
a fixed point of
\[
M_t(\policy) = \frac{\sum_{\phi \in \Phi}\ylink[\phi]{t}[\phi](\policy)}{\sum_{\phi \in \Phi}\ylink[\phi]{t}},
\]
when $R^\Phi_{t-1} \in \reals^{|\Phi|}_+ \setminus \{0\}$,
where $\ylink{t} \as (\ylink[\phi]{t})_{\phi \in \Phi} \as f(R^\Phi_{t-1})$,
 and arbitrarily otherwise.
Note that $M_t$ is a convex combination of linear
operators $\{[\phi]\}_{\phi \in \Phi}$, hence the fixed point always exists by the Brouwer Fixed Point Theorem.
If $\Phi = \Phi_{EXT}$ then the fixed point of $M_t$ is a distribution
$\policy \propto \ylink{t}$~\cite{greenwaldtech}.
Examples  of $(\Phi, f)$-regret-matching algorithms
include Hart's algorithm~\cite{Hart00Rm}---typically called ``regret-matching''---and Hedge~\cite{freund1997decision},
with link functions $f(x)_i = x_i^+$ and
$f(x)_i = e^{\frac{1}{\tau}x_i}$ with temperature parameter $\tau > 0$, respectively.

A useful technique for bounding regret when estimates are used in place of true values
is to define
an $\epsilon-$Blackwell condition, as was used in the RRM Theorem~\citep{waugh2015solving}.
The analysis in RRM was specific to $\Phi = \Phi_{EXT}$ and the
polynomial link $f$ with $p=2$.
To generalize across different link functions and $\Phi \subseteq \Phi_{ALL}$ we define the
\epsGenBlackwellCond/.

\begin{definition}[\epsGenBlackwellCondTitle/] For a given reward system
$(A, \mathcal{R})$, finite set of action transformations $\Phi \subseteq \Phi_{ALL}$,
and link function $f: \reals^{|\Phi|} \to \reals^{|\Phi|}_+$,
a learning algorithm satisfies the \epsGenBlackwellCond/
if
$
    f(R^\Phi_{t-1}(h)) \cdot \expectation_{a \sim \alg}[\rho^\Phi(a,r)] \leq \epsilon.
$
\end{definition}

The Regret Matching Theorem~\cite{greenwald2006bounds} shows that the
$(\Phi, f)$-Blackwell condition ($\epsilon =0$) holds with
equality for $(\Phi, f)$-regret-matching algorithms.

We seek to bound the expected average $\Phi$-regret
when an algorithm at time $t$ chooses the fixed point of
$\Tilde{M}_t \as \nicefrac{\sum_{\phi \in \Phi}{\ylinkEst[\phi]{t} [\phi]}}{\sum_{\phi \in \Phi}{\ylinkEst[\phi]{t}}}$,
when $\Tilde{R}^\Phi_{t-1} \in \reals^{|\Phi|}_+ \setminus \{0\}$ and arbitrarily otherwise,
where $\ylinkEst{t} \as f(\Tilde{R}^\Phi_{t-1})$ and $\Tilde{R}^\Phi_{t-1}$ is an
estimate of $R^\Phi_{t-1}$, possibly from a function approximator.
Such an algorithm is referred to as \defword{approximate $(\Phi,f)$-regret-matching}.

Similarly to the RRM Theorem~\citep{waugh2015solving, morrill2016}, we show that the $\epsilon$ parameter of the \epsGenBlackwellCond/
depends on the link output approximation error,
$\norm{\ylink{t}- \ylinkEst{t}}_1$.

\begin{theorem}\label{thm-epsbound}
Given reward system (A,$\mathcal{R}$), a finite set of action
transformations $\Phi \subseteq \Phi_{ALL}$, and link function
$f: \reals^{|\Phi|} \to \reals^{|\Phi|}_+$, then an
approximate ($\Phi,f$)-regret-matching algorithm,
$\{L_t\}_{t=1}^\infty$, satisfies the \epsGenBlackwellCondTitle/ with $\epsilon \leq
2 \supReward \norm{\ylink{t} - \ylinkEst{t}}_1$, where $\ylink{t} \as f(R^\Phi_{t-1})$, and
$\ylinkEst{t} \as f(\Tilde{R}^\Phi_{t-1})$.
\end{theorem}
\textbf{Omitted proofs are deferred to the appendix.}

For a $(\Phi,f)$-regret-matching algorithm, an approach to bound
the expected average $\Phi$-regret is to use the $(\Phi, f)$-Blackwell condition
along with a bound on
$\expectation[G(R^\Phi_t)]$ for an appropriate function $G$~\citep{greenwald2006bounds,cesa2006prediction}.
Bounding the regret for an approximate $(\Phi, f)$-regret-matching algorithm
will be done similarly, except the
bound on $\epsilon$ from Theorem \ref{thm-epsbound} will be used.
Proceeding in this fashion yields the following theorem:
\begin{theorem}\label{thm-expectationbound}
Given a real-valued reward system $(A, \mathcal{R})$ a finite set
    $\Phi \subseteq \Phi_{ALL}$ of action transformations. If
    $\langle G, g, \gamma \rangle$ is a Gordon triple\footnote{A Gordon triple $\langle G, g, \gamma \rangle$ consists of
     three functions $G : \reals^n \to \reals$, $g: \reals^n
     \to \reals^n$, and $\gamma : \reals^n \to \reals$
     such that for all $x,y \in \reals^n$,
     $G(x+y) \leq G(x) + g(x) \cdot y + \gamma(y)$.}, then an
    approximate $(\Phi, g)$-regret-matching algorithm $\{L_t \}_{t=1}^\infty$
    guarantees at all times $t \geq 0$
    \[
        \mathbb{E}[G(R^\Phi_t)] \leq G(0) +
        t \underset{a \in A, r \in \Pi}{\sup} \gamma(\rho^{\Phi}(a,r)) +
        2\supReward \sum_{s=1}^t{\norm{g(R^\Phi_{s-1})-g(\Tilde{R}^\Phi_{s-1})}_1}.
    \]
\end{theorem}

\subsection{Bounds for Specific Link Functions}
\label{sec:bounds}
\subsubsection{Polynomial}
Given the polynomial link function $f(x)_i = (x_i^+)^{p-1}$ we consider two cases $2 < p < \infty$
and $1 < p \leq 2$.
For the following results it is useful to denote
the maximal activation
$\mu(\Phi) = \mbox{max}_{a \in A}|\{\phi \in \Phi : \phi(a)\neq \delta_a \}|$ \cite{greenwald2006bounds}.

For the case $p>2$ we have the following bound on the expected average
$\Phi$-regret:
\begin{theorem}
\label{thm:largePoly}
Given an ODP, a finite set of action transformations
$\Phi \subseteq \Phi_{ALL}$, and the polynomial link
function $f$ with $p>2$, then an approximate $(\Phi,f)$-
regret-matching algorithm guarantees
\begin{align*}
      &\expectation\left[\underset{\phi \in \Phi}{\normalfont \mbox{max}}\frac{1}{t} R^\phi_t \right]
       \leq\\ &\frac{1}{t}\sqrt{t(p-1)4U^2(\mu(\Phi))^{2/p} +
        2 U \sum_{k=1}^t \norm{g(R^\Phi_{k-1}) - g(\Tilde{R}^\Phi_{k-1})}_1},
\end{align*}
where $g: \reals^{|\Phi|} \to \reals^{|\Phi|}_+$ and
$g(x)_i = 0 $ if $x_i \leq 0$,
$g(x)_i = \frac{2(x_i)^{p-1}}{\norm{x^+}^{p-2}_p}$ otherwise.
\end{theorem}

Similarly for the case $1 < p \leq 2$
we have the following.
\begin{theorem}
\label{thm:rrm}
Given an ODP, a finite set of action transformations
$\Phi \subseteq \Phi_{ALL}$, and the polynomial link
function $f$ with $1 < p \leq2$, then an approximate $(\Phi,f)$-
regret-matching algorithm guarantees
        \[
      \expectation\left[\underset{\phi \in \Phi}{\normalfont \mbox{max}}\frac{1}{t} R^\phi_t \right]
       \leq \frac{1}{t}\left(t(2U)^p\mu(\Phi) +
        2 U \sum_{k=1}^t \norm{g(R^\Phi_{k-1}) - g(\Tilde{R}^\Phi_{k-1})}_1 \right)^{1/p}
    \]
where $g: \reals^{|\Phi|} \to \reals^{|\Phi|}_+$ and $g(x)_i= p(x^+_i)^{p-1}$.
\end{theorem}

In comparison to the RRM Theorem~\citep{morrill2016}, the above bound is
tighter as there is no $\sqrt{|A|}$ term in front of the errors and the $|A|$ term
has been replaced by\footnote{For $\Phi =\Phi_{EXT}, \mu(\Phi)=|A|-1$.} $|A|-1$.
These improvements are due
to the tighter bound in Theorem \ref{thm-epsbound} and
the original $\Phi$-regret analysis~\cite{greenwald2006bounds}, respectively. Aside from these differences, the bounds
coincide.

\subsubsection{Exponential}
\begin{theorem}\label{thm-explink}
Given an ODP, a finite set of action transformations
$\Phi \subseteq \Phi_{ALL}$, and an exponential link
function $f(x)_i = e^{\frac{1}{\tau} x_i}$ with $\tau > 0$, then an approximate $(\Phi,f)$-
regret-matching algorithm guarantees
       \[
      \expectation\left[\underset{\phi \in \Phi}{\normalfont \mbox{max}}\frac{1}{t} R^\phi_t \right]
       \leq \frac{1}{t}\left(
       \tau\normalfont \mbox{ln}|\Phi| + 2 U \sum_{k=1}^t \norm{g(R^\Phi_{k-1}) - g(\Tilde{R}^\Phi_{k-1})}_1
       \right)+ \frac{2U^2}{\tau}
    \]
where $g: \reals^{|\Phi|} \to \reals^{|\Phi|}_+$ and $g(x)_i= e^{\frac{1}{\tau}x_i}/\sum_j{e^{\frac{1}{\tau}x_j}}$.
\end{theorem}

The Hedge algorithm corresponds to the exponential link function $f(x)_i = e^{\frac{1}{\tau}x_i}$ when $\Phi = \Phi_{EXT}$, so Theorem \ref{thm-explink} provides a bound on a regression Hedge algorithm. Note that in this case, the approximation error term is not inside a root function as it is under the polynomial link function. This seems to imply that at the level of link outputs, polynomial link functions have a better dependence on the approximation errors. However, $g$ in the exponential link function bound is normalized to the simplex while the polynomial link functions can take on larger values. So which link function has a better dependence on the approximation errors depends on the magnitude of the cumulative regrets, which depends on the environment and the algorithm's empirical performance.

\section{Extensive-Form Games}
\subsection{Background}

A \defword{zero-sum extensive-form game (EFG)} is a tuple
\[
\tuple{\Histories, \ActionSet, \Actions, \playerChoice, \chancePolicy, \InfoStates, \reward_1}.
\]
$\Histories$ is the set of valid action sequences and chance outcomes called \defword{histories}, where an action is an element of $\ActionSet$, and the set of actions available at each history is determined by $\Actions : \Histories \to \ActionSet$. The player to act (including the chance ``player'', $\chance$) at each non-terminal history is determined by $\playerChoice : \Histories \setminus \TerminalHistories \to \set{1, 2, \chance}$, where terminal histories are those with no valid actions, $\TerminalHistories \as \set{ \history | \history \in \Histories, \Actions(h) = \emptyset}$. $\chancePolicy$ is a fixed stochastic policy assigned to the chance player that determines the likelihood of random outcomes, like those from die rolls or draws from a shuffled deck of cards. $\InfoStates \as \InfoStates_1 \cup \InfoStates_2$ is the information partition and it describes which histories players can distinguish between. The set of histories where player $\player \in \set{1, 2}$ acts, $\Histories_{\player} \as \set{h | \playerChoice(h) = \player}$, are partitioned into a set of \defword{information states}, $\InfoStates_{\player}$, where for each information state $\infoState \in \InfoStates_{\player}$, $\infoState \subseteq \Histories_{\player}$, is a set of histories indistinguishable to $\player$.
Since $A(\history) = A(\history')$ if $\history, \history' \in
\infoState \in \InfoStates$, we can denote the actions at $\infoState$ as $A(\infoState)$.
We require \defword{perfect recall} so that for all histories in an information state, the sequence of information states admitted by the preceding histories must be identical. $\reward_1 : \TerminalHistories \rightarrow \reals$ is a \defword{reward} or \defword{utility} function for player 1. The game is zero sum because player 2's utility function
$\reward_2 \as -\reward_1$.

Player $\player$'s \defword{policy} or \defword{behavioral strategy},
$\policy_{\player} \in \PolicySpace_{\player}$
defines a probability distribution over valid actions at each of $\player$'s information states, and a \defword{joint policy} or \defword{strategy profile} is an assignment of policies for each player,
$\policy \as \tuple{\policy_1, \policy_2}$.
We use $\reachProb^{\policy}(z)$ to denote the probability of reaching terminal history $z \in \TerminalHistories$ under profile $\policy$ from the beginning of the game and $\reachProb^{\policy}(\history, z)$ the same except starting from history $\history \in \Histories$. We subscript $\reachProb$ by the player to denote that player's contribution to these probabilities
$\reachProb^{\policy}(z) = \reachProb_\player^{\policy}(z)\reachProb_{-\player}^{\policy}(z)$. The expected value to player $\player$ under profile $\policy$ is $\ev_{\player}(\policy) = \ev_{\player}(\policy_1, \policy_2) = \sum_{z \in \TerminalHistories} \reachProb^{\policy}(z) \reward_{\player}(z)$.

A \defword{best response} for player $\player$ to another player's strategy, $\policy_{-\player}$, is a policy that achieves the maximum reward against
$\policy_{-\player}$, $\ev^*_{\player}(\policy_{-\player}) = \max_{\policy_{\player} \in \PolicySpace_{\player}} \ev_{\player}\subex{\tuple{\policy_{\player}, \policy_{-\player}}}$.
A profile, $\policy$, is an \defword{$\gap$-Nash equilibrium} if neither player can
unilaterally deviate from their assigned policy and gain more than $\gap$. That is, if
$
\ev_{\player}(\policy) + \gap_{\player} \ge \ev^*_{\player}(\policy_{-\player}),
$
for each player $\player \in \set{1, 2}$,
then $\policy$ is a $\max\set{\gap_1, \gap_2}$-equilibrium, and the smallest approximation error is achieved when $\ev_{\player}(\policy) + \gap_{\player} = \ev^*_{\player}(\policy_{-\player})$. Therefore, in a zero-sum game, all strategies that are part of $\gap$-Nash equilibria are at most $\gap$-utility away from being minimax optimal.

Since the game is zero-sum, the average of best response values is equal to $\nicefrac{(\gap_1 + \gap_2)}{2}$. This is the \defword{exploitability} of the profile $\policy$. We use profile exploitability to measure equilibrium approximation error.

The exploitability of a profile $\policy$ is related to $R^{\Phi_{EXT}}$
(abbreviated to $R^{EXT}$) in a fundamental way.
First consider the induced normal form of an EFG, where actions taken by a player
consists of specifying an action at each information state. That is, from an ODP perspective,
the set of actions available to player $\player$ (the learning algorithm) is
$\tilde{\Actions} = \prod_{\infoState \in \InfoStates_{\player}}{\Actions(s)}$.
We can then define the expected regret at time $t$ for player $i$ with respect to action $a' \in \tilde{\Actions}$ when selecting a policy $\policy \in \Delta(\tilde{A})$
as the difference
\[
\regret^t_{i, a'} \as \expectation_{a \sim \policy} \subblock{
    \ev_i(a', \policy_{-\player}) - \ev(a, \policy_{-\player})
}.
\]
We can then define the cumulative external regret of player $i$ at time $t$ as
$\FullRegret_{i,t} \as \max_{a^* \in \tilde{\Actions}} \sum_{k=1}^t \regret^k_{i, a^*}$.
Note that this is an instance of an ODP problem where the sequence of reward functions
for player $i$ is induced by the opponent's sequence of policies. Furthermore,
the external regret defined here is with respect to player $i$'s expected reward (\ie/, interchanging the expectation and maximum in the previously described $\Phi$-regret objective).
The connection between $\FullRegret_{i,t}$ and Nash equilibria then follows from
the well-known folk theorem.
\begin{theorem}
\label{thm:folk}
If two ODPs are enmeshed so that the rewards of the learners always sum to zero and the action of one learner influences the reward function of the other, then they represent a repeated zero-sum game. If neither learner, $\player \in \set{1, 2}$, suffers more than $\gap_{\player}$
external regret after $t$-rounds,
$\frac{1}{t}\FullRegret_{\player, t} \le \gap_{\player}$,
then the profile formed from their average policies,
$\bar{\policy}_{\player, t} = \frac{1}{t} \sum_{k = 1}^t \policy_{\player, k}$,
is an $(\gap_1 + \gap_2)$-Nash equilibrium.
\end{theorem}
See, for example, Blum and Mansour~\cite{Blum07} for a proof.

\subsection{Counterfactual Regret Minimization}

The idea of \defword{counterfactual regret minimization (CFR)}~\cite{zinkevich2008regret} is that we can decompose an EFG into multiple ODPs, one at each information state.
We define the reward for action $a \in \Actions(\infoState)$ in the ODP at $\infoState \in \InfoStates_{\player}$ as the \defword{counterfactual value} of playing $a$, which is the expected value of playing $a$ assuming that player $\player$ plays to reach $\infoState$. Formally,
\[
\cfq^{\policy}_{\player}(\infoState, a) = \sum_{\history \in \infoState, z \in \TerminalHistories} \reachProb^{\policy}_{\player}(\history a, z) \reachProb_{-\player}^{\policy}(z) \reward_{\player}(z),
\]
where $\history a \in \Histories$ is the history that results from taking action $a$ at history $\history$, and
$\reachProb^{\policy}_{\player}(\history, z) = 0$ whenever $z$ is unreachable from $\history a$.

Accordingly, the regret, also referred to as
instantaneous regret, of the ODP learner at $\infoState \in \InfoStates_{\player}$ for not committing to $a \in \Actions(\infoState)$ is
\[
\cfr^{\policy}_{\player}(\infoState, a) = \cfq^{\policy}_{\player}(\infoState, a) - \sum_{a' \in \Actions(\infoState)} \policy_{\player}(\infoState, a') \cfq^{\policy}_{\player}(\infoState, a').
\]
We denote the cumulative counterfactual regret of an information state $\infoState$ and action $a$ as $\Cfr_{\player, t}(\infoState, a) = \sum_{k = 1}^t \cfr^{\policy^k}_{\player}(\infoState, a)$, where we denote the profile at time $k$ as $\policy^k \as \tuple{\policy_{1, k}, \policy_{2, k}}$, and that of $\infoState$ alone as $\Cfr_{\player, t}(\infoState) = \max_{a \in \Actions(\infoState)} \Cfr_{\player, t}(\infoState, a)$.

Zinkevich \etal/~\cite{zinkevich2008regret} showed:
\begin{theorem}[CFR]
\label{thm:cfr}
For both players, $\player \in \set{1, 2}$, the regret of $\player$'s policies constructed from their ODP learners after $t$ iterations of CFR is $\frac{1}{t}\FullRegret_{\player, t} \le \gap_{\player, t}$ where
$\gap_{\player, t} = \frac{1}{t} \sum_{\infoState \in \InfoStates_{\player}} \left(\Cfr_{\player, t}(\infoState)\right)^+$.
Furthermore, the profile of average sequence weight policies, $\bar{\policy}^t \as \tuple{\bar{\policy}_{1, t}, \bar{\policy}_{2, t}}$, is an $(\gap_{1, t} + \gap_{2, t})$-Nash equilibrium,
where
\[
\bar{\policy}_{\player, t}(\infoState) \propto \sum_{k = 1}^t \sum_{\history \in \infoState} \reachProb^{\policy^k}_{\player}(\history) \policy_{\player, k}(\infoState).
\]
\end{theorem}
See Farina \etal/~\cite{farina2019regret} for the sketch of an alternative proof using the regret circuits framework that is perhaps more intuitive than the proof in the original work.

\subsection{$f$-RCFR}

Games that humans are interested in playing, or those that model problems of practical importance, typically have an immense number of information states or actions. But such games often contain structure that can be recovered by endowing information state-action pairs (sequences) with a \defword{feature representation}, $\featureExp : \InfoStates \times \ActionSet \to \FeatureSpace, d > 0$. A function approximator,
$\functionApproximator : \FeatureSpace \to \reals$, could then make use of shared properties between sequences to allow more efficient learning.
RCFR~\cite{waugh2015solving} uses a function approximator to predict cumulative counterfactual regrets at each information state and generates policies with a normalized ReLU transformation.

Thanks to our new analysis of approximate regret matching, we now know that any link function that admits a no-$\Phi_{\text{EXT}}$-regret regret matching algorithm also has an approximate version. Rather than restricting ourselves to the polynomial link function with parameter $p = 2$, we can consider alternate parameter choices or alternative link functions, like the exponential function.
So instead of a normalized ReLU policy, we employ a policy generated by the external regret fixed point of link function $f: \reals^{\abs{\ActionSet}} \to \reals^{\abs{\ActionSet}}_+$ with respect to approximate regrets predicted by a functional regret estimator, $\Tilde{\Regret}(\infoState) = \subex{ \functionApproximator\subex{ \featureExp(\infoState, a) } }_{a \in \Actions(\infoState)}, \text{for all } \infoState \in \InfoStates$.
More formally, the $f$-RCFR policy for player $\player$ given functional regret estimator $\Tilde{\Regret}$ is
$
\policy(\infoState) \propto f(\Tilde{\Regret}(\infoState))
$
when $\Tilde{\Regret}(\infoState) \in \reals^{|\Actions(\infoState)|}_+ \setminus \{0\}$ and arbitrarily otherwise, for all $\infoState \in \InfoStates_{\player}$.
Since the input to any link function in an approximate regret matching algorithm is simply an estimate of the counterfactual regret, we can reuse all of the techniques previously developed for RCFR-like methods to train regret estimators~\cite{waugh2015solving,morrill2016,deepCFR,li2018doubleNeuralCfr,steinberger2019single}.

Using Theorem \ref{thm-expectationbound} and the CFR Theorem \ref{thm:cfr}, we can derive an improved regret bound with the polynomial link and a new bound with the exponential link.

\begin{corollary}[polynomial $(p > 2)$]\label{thm:largePolyRcfr}
Given the polynomial link function $f$ with $p>2$, let $\policy_{\player, k}(\infoState) \propto f(\Tilde{\Regret}_k(\infoState))$ be the policy that $f$-RCFR assigns to player $\player$ at iteration $k$ in information state $\infoState \in \InfoStates_{\player}$ and denote the cumulative approximation error in $\infoState$ as
$
\altGap_{\player}(\infoState) = \sum_{k=1}^t \norm{g(\Cfr_{k-1}(\infoState)) - g\subex{ \Tilde{\Cfr}_{k-1}(\infoState) }}_1,
$
where $g: \reals^{\abs{\Actions(\infoState)}} \to \reals^{\abs{\Actions(\infoState)}}_+$ and $g(x)_i = 0$ if $x_i \leq 0$, $g(x)_i= \frac{2(x_i)^{p-1}}{\norm{x^+}^{p-2}_p}$ otherwise.
Then after $t$-iterations, $f$-RCFR guarantees, for both players, $\player \in \set{1, 2}$,
$\frac{1}{t} \FullRegret_{\player, t} \leq \gap_{\player, t}$, where
\begin{align*}
  &\gap_{\player, t} = \frac{1}{t} \sum_{\infoState \in \InfoStates_{\player}} \sqrt{t(p-1)4U^2(\abs{\Actions(\infoState)} - 1)^{2/p} + 2 U \altGap_{\player}(\infoState)}.
\end{align*}
Noticing that $ \abs{\Actions(\infoState)} \leq \abs{\ActionSet}$ and letting
$\altGap^*_{\player} = \max_{\infoState \in \InfoStates_{\player}} \altGap_{\player}(\infoState)$, we have
\begin{align*}
  &\gap_{\player, t} \leq \frac{1}{t} \abs{\InfoStates_{\player}} \sqrt{t(p-1)4U^2(\abs{\ActionSet} - 1)^{2/p} + 2 U \altGap^*_{\player}}.
\end{align*}
Furthermore, the profile of average sequence weight policies, $\bar{\policy}^t$, is an $(\gap_{1, t} + \gap_{2, t})$-Nash equilibrium.
\end{corollary}

\begin{proof}
This result follows directly from Theorem \ref{thm:cfr}.
The counterfactual regret, $\Cfr_{\player, t}(\infoState)$, at each information state
corresponds to $\Phi_{EXT}$ regret for an online ODP with $\mu(\Phi_{EXT}) = \abs{A(\infoState)} - 1$.
Therefore, playing an approximate $(\Phi_{EXT}, f)$-regret matching
algorithm at each state with a polynomial link funciton with $p>2$
results in the regret bound presented in Theorem \ref{thm:largePoly}
for each state specific ODP.
Although Theorem \ref{thm:largePoly} is stated with respect to random regrets and
counterfactual regret is an expected regret, the analysis
of Greenwald \etal/~\cite[Corollary 18]{greenwald2006bounds}
allows us to trivially extend our bounds from Section \ref{sec:bounds} to this
case.
The result then follows trivially from Theorem \ref{thm:cfr}.
\end{proof}

The proofs for the polynomial link with $1 < p \leq 2$ and the exponential
link are very similar and omitted for brevity.

\begin{corollary}[polynomial $(1 < p \leq 2)$]\label{thm:smallPolyRcfr}
Given the polynomial link function $f$ with $1 < p\leq2$, let $\policy_{\player, k}(\infoState) \propto f(\Tilde{\Regret}_k(\infoState))$ be the policy that $f$-RCFR assigns to player $\player$ at iteration $k$ in information state $\infoState \in \InfoStates_{\player}$ and denote the cumulative approximation error in $\infoState$ as
$
\altGap_{\player}(\infoState) = \sum_{k=1}^t \norm{g(\Cfr_{k-1}(\infoState)) - g\subex{ \Tilde{\Cfr}_{k-1}(\infoState) }}_1,
$
where $g: \reals^N \to \reals^N_+$, and $g(x)_i= p(x^+_i)^{p-1}$.
Then after $t$-iterations, $f$-RCFR guarantees, for both players, $\player \in \set{1, 2}$,
$\frac{1}{t} \Regret^{\text{EXT}}_{\player, t}\leq \gap_{\player, t}$, where
\begin{align*}
  &\gap_{\player, t} = \frac{1}{t} \sum_{\infoState \in \InfoStates_{\player}}
  \left(t(2U)^p(\abs{\Actions(\infoState)} - 1) + 2 U \altGap_{\player}(\infoState) \right)^{1/p}.
\end{align*}
Noticing that $ \abs{\Actions(\infoState)} \leq \abs{\ActionSet}$ and letting
$\altGap^*_{\player} = \max_{\infoState \in \InfoStates_{\player}} \altGap_{\player}(\infoState)$, we have
\begin{align*}
  &\gap_{\player, t} \leq \frac{1}{t} \abs{\InfoStates_{\player}}
  \left(t(2U)^p(\abs{\ActionSet} - 1) + 2 U \altGap^*_{\player} \right)^{1/p}.
\end{align*}
Furthermore, the profile of average sequence weight policies, $\bar{\policy}^t$, is an $(\gap_{1, t} + \gap_{2, t})$-Nash equilibrium.
\end{corollary}
The above theorem provides a tighter bound for RCFR ($p=2$) than what exists in the literature.
The improvement is a direct consequence of the tighter bound for RRM presented in
Theorem \ref{thm:rrm} in Section \ref{sec:bounds}.
Given the application of the RRM Theorem by Brown \etal/~\cite{deepCFR}, these results
should lead to a tighter bound when a function approximator
learns from sampled counterfactual regret targets.

\begin{corollary}[exponential]\label{thm:expRcfr}
Given the exponential link function $f$ with $\tau > 0$, let $\policy_{\player, k} \propto f(\Tilde{\Regret}_k(\infoState))$ be the policy that $f$-RCFR assigns to player $\player$ at iteration $k$ given functional regret estimator $\functionApproximator_k : \FeatureSpace \to \reals$ and denote the cumulative approximation error in $\infoState$ as
$
\altGap_{\player}(\infoState) = \sum_{k=1}^t \norm{g(\Cfr_{k-1}(\infoState)) - g\subex{ \Tilde{\Cfr}_{k-1}(\infoState) }}_1,
$
where $g: \reals^N \to \reals^N_+$, and $g(x)_i= e^{\frac{1}{\tau}x_i}/\sum_j{e^{\frac{1}{\tau}x_j}}$.
Then after $t$-iterations, $f$-RCFR guarantees, for both players, $\player \in \set{1, 2}$,
$\frac{1}{t} \Regret^{\text{EXT}}_{\player, t}\leq \gap_{\player, t}$, where
\begin{align*}
  &\gap_{\player, t} = \sum_{\infoState \in \InfoStates_{\player}}
  \left( \frac{1}{t}\left( \tau \ln \abs{\Actions(\infoState)} + 2 U \altGap_{\player}(\infoState) \right)+ \frac{2U^2}{\tau}\right).
\end{align*}
Noticing that $ \abs{\Actions(\infoState)} \leq \abs{\ActionSet}$ and letting
$\altGap^*_{\player} = \max_{\infoState \in \InfoStates_{\player}} \altGap_{\player}(\infoState)$, we have
\begin{align*}
  &\gap_{\player, t} \leq
  \abs{\InfoStates_i} \subex{
    \frac{1}{t} \left( \tau \ln \abs{\ActionSet} + 2 U \altGap^*_{\player} \right)+ \frac{2U^2}{\tau}
  }.
\end{align*}
Furthermore, the profile of average sequence weight policies, $\bar{\policy}^t$, is an $(\gap_{1, t} + \gap_{2, t})$-Nash equilibrium.
\end{corollary}

This bound shares the same advantage with respect to the action set size dependence  over the polynomial RCFR bounds as the bound of Theorem~\ref{thm-explink} has over the bounds of  Theorems~\ref{thm:rrm} and~\ref{thm:smallPolyRcfr}.

With the exponential link function, $f$-RCFR is approximately Hedge applied to each information state with function approximation. To make a connection with the field of reinforcement learning, we can compare $f$-RCFR with two recently developed algorithms that also generalize Hedge to sequential decision problems and utilize function approximation: \Politex/~\cite{abbasi2019politex} and neural replicator dynamics (NeuRD)~\cite{neuRD}.

In contrast to $f$-RCFR, \Politex/ trains models to predict cumulative action values.
An action value is proportional to a counterfactual value
where the constant depends on the policies of the other players and chance~\cite{zinkevich2008regret,CFR_Actor_Critic}. If \Politex/ instead trains on counterfactual regrets, then we arrive at an $f$-RCFR instance with a softmax parameterization and a regret estimator updated in a two-step process: construct an instantaneous regret estimator and combine it with the previous estimator to predict cumulative regrets.
In fact, our implementation of $f$-RCFR for the experiments that follow uses the same two-step update procedure.

Instead of training a model of instantaneous regrets, NeuRD performs a gradient descent step on the squared loss between the current policy logits and a target constructed by adding the logits to the instantaneous regret after each iteration.
We can see this as a ``bootstrap'' regret target, as described by Morrill~\cite{morrill2016}, where the policy logits are approximate regrets. NeuRD is therefore an instance of $f$-RCFR with a softmax parameterization and a regret estimator trained on bootstrap regret targets.

\section{Experiments}
To examine the impact of the link function, choices for their parameters, and the interaction between link function and function approximation, we test $f$-RCFR in two games commonly used as research testbeds, Leduc hold'em poker~\cite{Southey05leduc} and imperfect information goofspiel~\cite{lanctot13phdthesis} with linear function approximation.
\begin{figure}[t]
    \centering
    \includegraphics[trim={1.3cm 0 0 0}, width=0.8\linewidth]{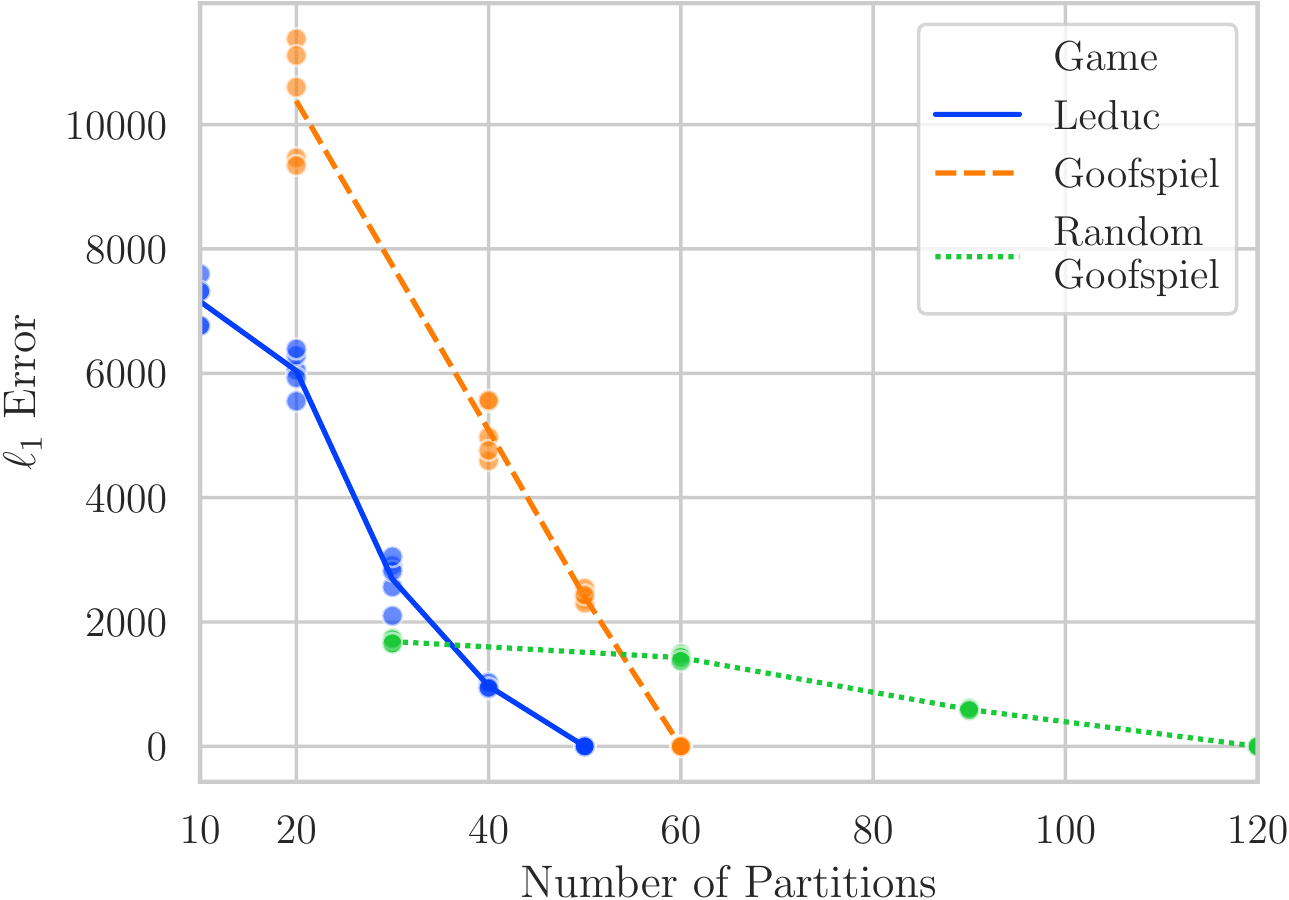}
    \caption{
    The cumulative counterfactual regret estimation error accumulated over time and information states for select $f$-RCFR instances in
    Leduc hold'em poker, goofspiel, and random goofspiel.
    For each game and setting of the number of partitions, we select the link
    function and the parameter with the smallest average exploitability over 5-runs at 100K-iterations.
    The solid lines connect the average error across iterations and dots show the errors of individual runs.}
    \label{fig:error}
\end{figure}

\begin{figure*}[hptb]
  \includegraphics[width=0.9\linewidth]{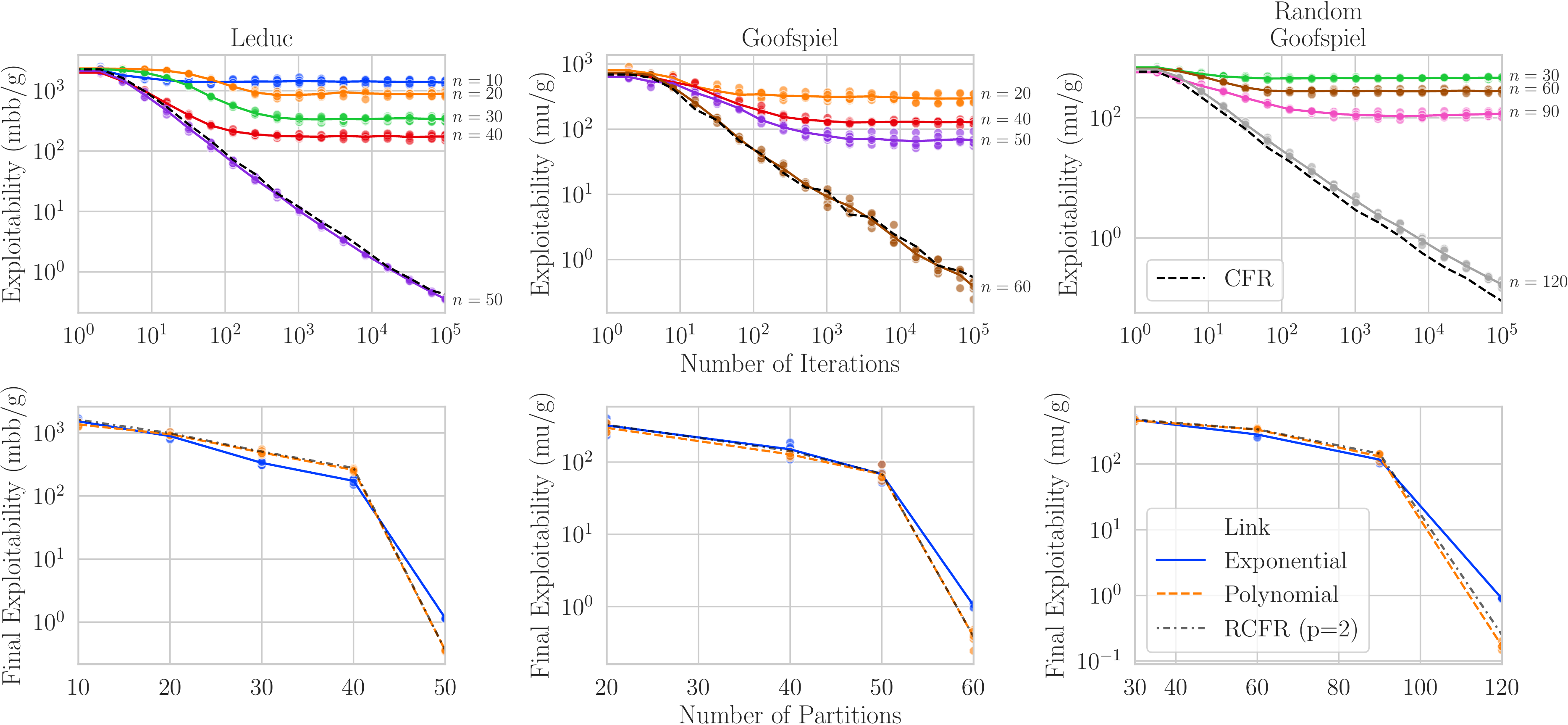}
  \caption{
    (top) The exploitability of the average strategy profile of tabular CFR and $f$-RCFR instances during the first 100K-iterations in Leduc hold'em (top left), goofspiel (top center), and random goofspiel (top right).
    For each setting of the number of partitions, we show the performance of the $f$-RCFR instance with the link function and parameter that achieves the
    lowest average final exploitability over 5-runs.
    The mean exploitability and the individual runs are plotted for the chosen instances as lines and dots respectively.
    (bottom) The final average exploitability after
    100K-iterations for the best exponential and polynomial link function instances in Leduc hold'em (left), goofspiel (center), and random goofspiel (right).
  }
\label{fig:performance}
\end{figure*}

\subsection{Algorithm Implementation}

Our regret estimators are independent linear function approximators for each player, $\player \in \set{1, 2}$, and action $a \in \bigcup_{\infoState \in \InfoStates_{\player}} \Actions(\infoState)$.
Our features are built on tug-of-war hashing features~\cite{bellemare2012sketch}.

We randomly partition the information states that share the same action into $m$-buckets and repeat this $n$-times to generate $n$-sparse indicator features of length $m$. The sign of each feature is randomly flipped to -$1$ independently to reduce bias introduced by collisions. The expected sign associated with all other information states that share a non-zero entry in their feature vector is, by design, zero. We use the number of partitions, $n$, to control the severity of approximation in our experiments.

We do ridge regression on counterfactual regret targets to train our regret estimators.
After the first iteration, we simply add this new vector of weights to our previous weights. Since the counterfactual regrets are computed for each information state-action sequence on every iteration, the same feature matrix is used during training after each iteration. Therefore, the ridge regression solution is a linear function of the targets and the sum of the optimal weights for predicting counterfactual regret yields the ridge regression solution weights for predicting the sum. Beyond training the weights at the end of each iteration, the regrets do not need to be saved or reprocessed.

Since we are most interested in comparing the performance of $f$-RCFR with different link functions and parameters, we track the average policies for each instance exactly in a table. While this is less practical than other approaches, such as learning the average policies from data, it removes another variable from the analysis and allows us to examine the impact of different link functions in relative isolation. Equivalently, we could have saved copies of the regret estimator weights across all iterations and computed the average policy on demand, similarly to Steinberger~\cite{steinberger2019single}.

\begin{figure}[bt]
    \centering
    \includegraphics[trim={2cm 0 0 0}, width=0.95\linewidth]{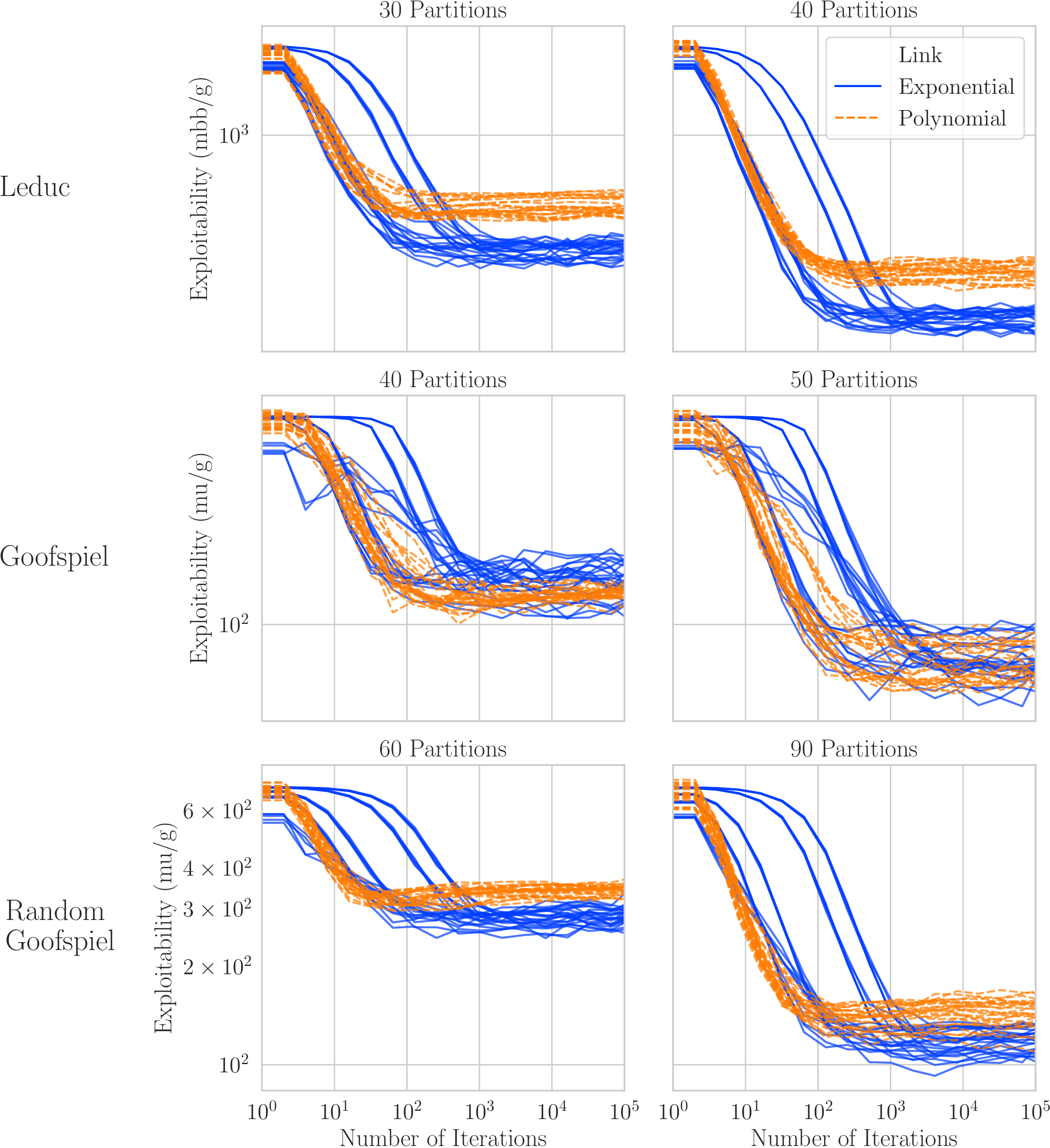}
    \caption{Exploitability of the average strategy profile for all configurations
    and runs with the exponential
    and polynomial link functions.
    The exponential link function achieves a lower exploitability than
    the polynomial link function
    when a moderate number of partitions (30 or 40) are used in Leduc hold'em (top).
    The same occurs in random goofspiel with 60 or 90-partitions (bottom).
    Both link functions perform similarly in goofspiel with 40 or 50-partitions (center).
    }
    \label{fig:plateau}
\end{figure}

\subsection{Games}
In Leduc hold'em poker~\citep{Southey05leduc}, the deck consists of 6 cards, two suits each with 3 ranks (\eg/, king,
queen, and ace), and played with two players. At the start of the game each player antes 1
chip and receives one private card. Betting is restricted to two rounds with a maximum of
two raises each round, and bets are limited to 2 and 4 chips. Before the second round of betting
a public card is revealed from the deck. Provided no one folds, the player
with a private card matching the public card wins, if no players match,
the winnings go to the player with the private card of highest rank. This game has 936 states.

Goofspiel is played with two players and a deck with three suits.
Each suit consists of $N$ cards of different rank. Two of the suits form the hands of the players. The third is used as a deck of point cards. At each round a card is revealed from the
point deck and players simultaneously bid by playing a card
from their hand. The player with
the highest bid (\ie/~highest rank)
receives points equal to the rank of the revealed card. The player with the most points when the point deck runs out is the winner and receives a utility of +1. The loser receives a utility of -1.
We use an imperfect information variant of goofspiel where the bidding cards are not revealed~\citep{lanctot13phdthesis}. We use two variants of goofspiel: one with a shuffled point deck and four ranks that we call ``random goofspiel'' and a second with a sorted point deck in decreasing order but five ranks that we call ``goofspiel''. Goofspiel is roughly twice as large as Leduc hold'em at 2124-information states, while random goofspiel is larger still at 3608-information states.

Our experiments use the \emph{OpenSpiel}~\citep{LanctotEtAl2019OpenSpiel} implementations of these games.

Convergence to a Nash equilibrium in each game is measured by the exploitability of the average strategy
profile after each $f$-RCFR iteration.
Exploitability in Leduc hold'em is measured in milli-big blinds. Exploitability in goofspiel and random goofspiel is measured in milli-utils.

\subsection{Parameters}

From Theorems \ref{thm:cfr} and \ref{thm:folk}, any network of external regret minimizers
(one at each information state) can be combined
to produce an average strategy profile with bounded exploitability. Therefore, the
bounds presented in Section \ref{sec:bounds} provide an exploitability bound for
$f$-RCFR algorithms where $f$ is a polynomial or exponential link function,
and estimates of counterfactual regrets are used at each information state
in place of true values (Corollaries \ref{thm:largePolyRcfr}, \ref{thm:smallPolyRcfr}, and \ref{thm:expRcfr}).

Most notably, the appearance of function approximator error within the regret bounds
in Section \ref{sec:bounds} appear in different forms depending on the link function
$f$. For the polynomial link function, the bounds vary with the $p$ parameter
and similarly the exponential link with the $\tau$ parameter.
We tested the polynomial link function with $p \in \set{1.1, 1.5, 2, 2.5, 3}$ to test values around the common choice ($p = 2$).
The exponential link function was tested with $\tau \in \set{0.01, 0.05, 0.1, 0.5, 1}$ in Leduc hold'em and random goofspiel, and
$\tau \in \set{0.1, 0.5, 1, 5, 10}$ in goofspiel.

To examine the relationships between a link function, link function specific parameters, and
function approximator error, we examine the empirical exploitability of $f$-RCFR
with different levels of approximation. The degree of approximation is adjusted via the quality of features. In particular, we vary the number of partitions, $n$. Increasing $n$ increases discriminative power and reduces approximation error (Figure \ref{fig:error}).

The number of buckets in each partition is fixed at $m=10$. If the number of information states that share an action is not evenly divisible by ten, a subset of the buckets are assigned one more information state than the others. Thus, adding a partition adds ten features. Only one feature per partition is non-zero for any given information set, so the prediction cost grows linearly with the number of partitions. The ridge regression update cost however, grows quadratically with the total number of features.

\subsection{Results and Analysis}

Figure \ref{fig:performance} shows the average exploitability of the best link function and hyper-parameter configuration during learning (top) and after 100k-iterations (bottom). The best parameterization was selected according to the average final exploitability after 100K-iterations over 5-runs.
Notice that the exploitability of the average strategy profile decreases as the number of partitions increases, as predicted by the $f$-RCFR exploitability bounds given the decrease in the prediction error associated with increasing the number of partitions (Figure \ref{fig:error}).

With 30 and 40-partitions in Leduc hold'em, and 60 and 90 in random goofspiel, the best instance with an exponential link function outperforms all of those with polynomial link
functions, including RCFR (polynomial link with $p = 2$) (Figure \ref{fig:plateau}, top and bottom). These feature parameters correspond to a moderate amount of function approximation error. In addition, this performance difference was observed across all configurations of the exponential and polynomial link in Leduc hold'em. \ie/, all of the instances with the exponential link function plateau to a final average exploitability lower than that of all those with  polynomial link functions.

The exponential link function does not outperform the polynomial link function in goofspiel or when the number of partitions is large, however (Figure \ref{fig:plateau}, center and Figure \ref{fig:performance}, bottom). Thus, the relative performance of different link functions is dependent on the game and the degree of function approximation error.

Among the different choices of $p$ for the polynomial link function, $p=2$ (RCFR) performs well with respect to the other polynomial instances across all partition numbers and in all three games
(Figure \ref{fig:performance} (bottom)). It is outperformed only by $p=1.1$ and $p=1.5$ in random goofspiel with many partitions, $n=90$ and $n=120$ respectively.

\section{Conclusions}
In this paper, we generalize the RRM Theorem in two dimensions---the link
function, including the polynomial and exponential link functions---and
regret metrics, including external and internal regret.
The generalization to different link functions allows us to construct
regret bounds for a general $f$-RCFR algorithm. The $f$-RCFR algorithm
can approximate Nash equilibria in zero-sum games with imperfect information
using alternative functional policy parameterizations beyond the
previously studied normalized ReLU parameterization.

We then examine the performance of $f$-RCFR with the polynomial and exponential link functions
under different hyper-parameter choices and different levels of function approximation error
in Leduc hold'em poker and imperfect information goofspiel.
$f$-RCFR with the polynomial link function and $p=2$ often achieves an exploitability competitive with or lower than other choices, but the exponential link function can outperform all polynomial parameters when the functional regret estimator has a moderate degree of approximation.

This work focuses primarily on the benefits of alternatives to the ReLU policy parameterization.
However, extending the RRM Theorem to a more general class of regret metrics
that includes internal regret also suggests future directions, particularly the
approximation of correlated equilibria~\cite{cesa2006prediction} or
extensive-form correlated equilibria~\cite{von2008efce} with function approximation.

NeuRD~\cite{neuRD} and Politex~\cite{abbasi2019politex} demonstrate that
benefits can be gained by adapting a regret-minimizing method to the
function approximation case in RL settings.
These algorithms are also particular ways of implementing approximate Hedge, utilizing softmax policies. Since ReLU policies outperform softmax policies in some cases, it would be worthwhile to investigate their performance in RL applications.

\section*{Acknowledgments}
\phantomsection\addcontentsline{toc}{section}{Acknowledgments}
We acknowledge the support of the Natural Sciences and Engineering
Research Council of Canada (NSERC), the Alberta Machine Intelligence Institute (Amii), and Alberta
Treasury Branch (ATB). Computing resources were provided by WestGrid and Compute Canada.

\bibliographystyle{ACM-Reference-Format}
\bibliography{ref}

\appendix
\section{Existing Results}
Below we recall results from Greenwald \etal/~\cite{greenwald2006bounds} and include the detailed proofs omitted in the main body of the
paper.

Many of the following results make use of a Gordon triple. We restate the definition
from Greenwald \etal/~below.

\begin{definition}\label{gordon_triple}
     A Gordon triple $\langle G, g, \gamma \rangle$ consists of
     three functions $G : \reals^n \to \reals$, $g: \reals^n
     \to \reals^n$, and $\gamma : \reals^n \to \reals$
     such that for all $x,y \in \reals^n$,
     $G(x+y) \leq G(x) + g(x) \cdot y + \gamma(y)$.
\end{definition}

\begin{lemma}\label{lemma:21}
    If $x$ is a random vector that takes values in $\reals^n$, then
    $(\expectation[\mbox{\normalfont max}_i x])^q
    \leq \expectation [\norm{x^+}^q_p]
    $
    for $p,q \geq 1$.
\end{lemma}
See [Lemma 21]\cite{greenwald2006bounds}.

\begin{lemma}\label{lemma:22}
    Given a reward system $(A, \mathcal{R})$
    and a finite set of action transformations
    $\Phi \subseteq \Phi_{ALL}$, then
    $\norm{\rho^\Phi(a,r)}_p \leq \radius (\mu(\Phi))^{1/p}$
    for any reward function $r \in \Pi$.
\end{lemma}
The proof is indentical to [Lemma 22]\cite{greenwald2006bounds}
except we have that regrets are bounded in $[-2U, 2U]$
instead of $[-1,1]$. Also note that by assumption
$\mathcal{R}$ is bounded.

\begin{theorem}[Gordon 2005]\label{thm:gordon}
Assume $\langle G, g, \gamma \rangle$ is a Gordon triple and
$C: \mathcal{N} \to \reals$. Let $X_0 \in \reals^n$,
let $x_1,x_2,...$ be a sequence of random vectors over $\reals^n$,
and define $X_t = X_{t-1}+x_t$ for all times $t \geq 1$. \\
If for all times $t \geq 1$,
\[
    g(X_{t-1})\cdot \expectation[x_t|X_{t-1}] +
    \expectation[\gamma(x_t)|X_{t-1}] \leq C(t) \quad a.s.
\]
then, for all times $t \geq 0$,
\[
    \expectation[G(X_t)] \leq G(X_0) + \sum_{\tau=1}^t{C(\tau)}.
\]
\end{theorem}
It should be noted that the above theorem was originally proved
by Gordon \cite{gordon2005no}.

\section{Proofs}

\subsection{Theorem \ref{thm-epsbound}}

\begin{proof}
We denote $r = (r^\prime(a))_{a \in A}$ as the reward vector for an arbitrary reward function
$r^\prime : A \to \reals$. Since by construction this algorithm chooses $L_t$ at each timestep $t$ to be the fixed point of $\Tilde{M}_t$, all that remains to be shown is that this algorithm satisfies the $(\Phi, f, \epsilon)$-Blackwell condition with
$\epsilon \leq 2 \supReward \norm{ \ylink[\Phi]{t} - \ylinkEst[\Phi]{t} }_1, t > 0$.

By expanding the value of interest in the $(\Phi,f)$-Blackwell condition and applying elementary upper bounds, we arrive at the desired bound. For simplicity, we omit timestep indices with $\algSimple \as \alg$, and similarly for $\ylink[\Phi]{t},
\ylinkEst[\Phi]{t}$.
First, suppose $\sum_{\phi \in \Phi}\ylinkEst[\phi]{} \neq 0$:
\begin{align*}
\ylinkSimple &\cdot \mathbb{E}_{a \sim \algSimple}[\rho^\Phi (a,r)]
= \sum_{\phi \in \Phi}{
  \ylinkSimple[\phi](r\cdot [\phi](\algSimple) - r \cdot \algSimple)
} \\
&= r \cdot \left(
  \sum_{\phi \in \Phi} \ylinkSimple[\phi] \left([\phi](\algSimple) - \algSimple\right)
\right)\\
&= r \cdot \left(
\sum_{\phi \in \Phi} \left(
  \ylinkEstSimple[\phi] - \ylinkEstSimple[\phi] +
  \ylinkSimple[\phi]
\right)\left(
  [\phi](\algSimple) - \algSimple
\right)
\right)\\
&=r \cdot \left(
  \left(\sum_{\phi \in \Phi} \ylinkEstSimple[\phi] \right)
  \left(\Tilde{M} \algSimple - \algSimple \right) +
  \sum_{\phi \in \Phi} (
    \ylinkSimple[\phi] - \ylinkEstSimple[\phi]
  )(
    [\phi](\algSimple) - \algSimple
  )
\right)\\
&=r \cdot \left(
  \sum_{\phi \in \Phi} (
    \ylinkSimple[\phi] - \ylinkEstSimple[\phi]
  )(
    [\phi](\algSimple) - \algSimple
  )
\right)\\
&\leq \norm{r}_\infty \norm{
     \sum_{\phi \in \Phi} (
      \ylinkSimple[\phi] - \ylinkEstSimple[\phi]
    )(
      [\phi](\algSimple) - \algSimple
    )
}_1 \\
&\leq \norm{r}_\infty
  \sum_{\phi \in \Phi} | \ylinkSimple[\phi] - \ylinkEstSimple[\phi] | (
    \norm{[\phi](\algSimple)}_1 +
    \norm{\algSimple}_1
  )\\
&\leq \norm{r}_\infty
  \sum_{\phi \in \Phi} | \ylinkSimple[\phi] - \ylinkEstSimple[\phi] |(1 + 1) \\
&\leq 2 \supReward
  \norm{ \ylinkSimple[\Phi] - \ylinkEstSimple[\Phi] }_1.
\end{align*}
If $\sum_{\phi \in \Phi}\ylinkEstSimple[\phi] = 0$ it is easy to see the inequality still holds.

Therefore, $\{L_t\}_{t=1}^\infty$ satisfies the $(\Phi, f, \epsilon)$-Blackwell condition with
$\epsilon \leq 2 \supReward \norm{ \ylinkSimple[\Phi] - \ylinkEstSimple[\Phi] }_1$, as required to complete the argument.
\end{proof}

An important observation of Theorem \ref{thm-epsbound} is the following corollary:
\begin{corollary}\label{cor:approxRmEpsBlackwell}
For a reward system $(A,\mathcal{R})$, finite set of
action transformations $\Phi \subseteq \Phi_{ALL}$, and two link functions $f$ and $f^\prime$,
if there exists a strictly positive function
$\psi : \reals^{|\Phi|}\to \reals$ such that $f^\prime(x) = \psi(x)f(x)$
then for any $\epsilon \in \reals$,
an approximate $(\Phi, f)$-regret-matching algorithm satisfies
\[
    f^\prime(R^\Phi_{t-1}(h)) \cdot \expectation_{a \sim \alg}[\rho^\Phi(a,r)] \leq 2 \supReward
                 \norm{f^\prime(R^\Phi_{t-1})-f^\prime(\Tilde{R}^\Phi_{t-1})}_1.
\]
\end{corollary}
\begin{proof}
    The reasoning is similar to [Lemma 20]\cite{greenwald2006bounds}. The played fixed point is the
    same under both link functions, thus following the same
    steps to Theorem \ref{thm-epsbound} provides the above bound.
\end{proof}

\subsection{Proof of Theorem \ref{thm-expectationbound}}

The proof is similar to [Corollary
7]\cite{greenwald2006bounds} except that the learning
algorithm is playing the approximate fixed point with respect to the link function $g$.
From Theorem \ref{thm-epsbound} we have $g(R^\Phi_{t-1}(h))\cdot \expectation_{a \sim \alg}[\rho^\Phi(a,r)] \leq 2U\norm{g(R^\Phi_{t-1})-g(\Tilde{R}^\Phi_{t-1})}_1$.
Noticing that
\[
\expectation_{a \sim \alg}[\rho^\Phi(a,r)] =
\expectation[\rho^\Phi(a,r)|R^\Phi_{t-1}]
\]
and taking
$x_t = \rho^\Phi(a,r), X_t=R^\Phi_t$ we have
\begin{align*}
    &g(X_{t-1})\cdot \expectation[x_t|X_{t-1}] +
\expectation[\gamma(x_t)|X_{t-1}] \leq \\
&2U\norm{g(R^\Phi_{t-1})-g(\Tilde{R}^\Phi_{t-1})}_1 +
\underset{a \in A, r \in \Pi}{\sup} \gamma(\rho^{\Phi}(a,r)).
\end{align*}
The result directly follows from Theorem \ref{thm:gordon} by taking
\[
C(\tau) = 2U\norm{g(R^\Phi_{\tau-1})-g(\Tilde{R}^\Phi_{\tau-1})}_1 +
    \underset{a \in A, r \in \Pi}{\sup} \gamma(\rho^{\Phi}(a,r)). \qed{}
\]

\subsection{Proof of Theorem \ref{thm:largePoly}}

The proof follows closely to [Theorem 9]\cite{greenwald2006bounds}.
Taking $G(x) = \norm{x^+}^2_p$ and $\gamma(x)=(p-1)\norm{x}^2_p$ then
$\langle G, g, \gamma \rangle$ is a Gordon triple \cite{greenwald2006bounds}.
Given the above Gordon triple we have
\begin{align*}
       &\left( \expectation\left[
        \underset{\phi \in \Phi}{\mbox{max}}R^\phi_t
    \right] \right)^2 \leq \expectation \left[\norm{(R^\Phi_t)^+}^2_p\right]\\
    &= \expectation[G(R^\Phi_t)] \\
    &\leq G(0) +
    t \underset{a \in A, r \in \Pi}{\sup} \gamma(\rho^{\Phi}(a,r)) +
        2\supReward \sum_{s=1}^t{\norm{g(R^\Phi_{s-1})-g(\Tilde{R}^\Phi_{s-1})}_1} \\
    &\leq
    G(0)  + t(p-1)4U^2(\mu(\Phi))^{2/p}
    + 2 U \sum_{k=1}^t \norm{g(R^\Phi_{k-1}) - g(\Tilde{R}^\Phi_{k-1})}_1.
\end{align*}
The first inequality is from Lemma \ref{lemma:21}. The second inequality follows
from Corollary \ref{cor:approxRmEpsBlackwell} and Theorem \ref{thm-expectationbound}.
The third inequality is an application of Lemma \ref{lemma:22}. The result
then immediately follows. \qed{}

\subsection{Proof of Theorem \ref{thm:rrm}}

The proof follows closely to [Theorem 11]\cite{greenwald2006bounds}.
Taking $G(x) = \norm{x^+}^p_p$ and $\gamma(x)=(p-1)\norm{x}^p_p$ then
$\langle G, g, \gamma \rangle$ is a Gordon triple~\cite{greenwald2006bounds}.
Given the above Gordon triple we have
\begin{align*}
       &\left( \expectation\left[
        \underset{\phi \in \Phi}{\mbox{max}}R^\phi_t
    \right] \right)^p \leq \expectation \left[\norm{(R^\Phi_t)^+}^p_p\right]\\
    &= \expectation \left[G(R^\Phi_t) \right] \\
    &\leq G(0) +
    t \underset{a \in A, r \in \Pi}{\sup} \gamma(\rho^{\Phi}(a,r)) +
        2\supReward \sum_{s=1}^t{\norm{g(R^\Phi_{s-1})-g(\Tilde{R}^\Phi_{s-1})}_1} \\
    &\leq
    G(0)  + t (2U)^p(\mu(\Phi))
    + 2 U \sum_{k=1}^t \norm{g(R^\Phi_{k-1}) - g(\Tilde{R}^\Phi_{k-1})}_1.
\end{align*}
The first inequality is from Lemma \ref{lemma:21}. The second inequality follows
from Corollary \ref{cor:approxRmEpsBlackwell} and Theorem \ref{thm-expectationbound}.
The third inequality is an application of Lemma \ref{lemma:22}. The result
then immediately follows. \qed{}

\subsection{Proof of Theorem \ref{thm-explink}}

The proof follows closely to [Theorem 13]\cite{greenwald2006bounds}.
Taking $G(x) = \tau\mbox{ln}\left(\sum_i{e^{\frac{1}{\tau} x_i}} \right)$ and $\gamma(x)=\frac{1}{2 \tau}\norm{x}^2_\infty$ then
$\langle G, g, \gamma \rangle$ is a Gordon triple \cite{greenwald2006bounds}.
Given the above Gordon triple we have
\begin{align*}
       &\expectation\left[
        \underset{\phi \in \Phi}{\mbox{max }}\frac{1}{\tau} R^\phi_t
    \right]  = \expectation \left[
         \mbox{ln }e^{\underset{\phi \in \Phi}{\mbox{max }} \frac{1}{\tau} R^\phi_t}
    \right]\\
    &= \expectation \left[
         \mbox{ln } \underset{\phi \in \Phi}{\mbox{max }} e^{\frac{1}{\tau} R^\phi_t}
    \right]\\
    &\leq \expectation \left[
         \mbox{ln } \sum_{\phi \in \Phi} e^{ \frac{1}{\tau} R^\phi_t}
    \right]\\
    &= \frac{1}{\tau} \expectation[G(R^\Phi_t)] \\
    &\leq \frac{1}{\tau} \left( G(0) +
    t \underset{a \in A, r \in \Pi}{\sup} \gamma(\rho^{\Phi}(a,r)) +
        2\supReward \sum_{s=1}^t{\norm{g(R^\Phi_{s-1})-g(\Tilde{R}^\Phi_{s-1})}_1} \right) \\
    &\leq \frac{1}{\tau} \left(G(0) +
    t\frac{2U^2}{\tau} +
        2\supReward \sum_{s=1}^t{\norm{g(R^\Phi_{s-1})-g(\Tilde{R}^\Phi_{s-1})}_1} \right).
\end{align*}
The second inequality follows
from Corollary \ref{cor:approxRmEpsBlackwell} and Theorem \ref{thm-expectationbound}.
The result then immediately follows. \qed{}

\end{document}